\documentclass{article}

\usepackage{PRIMEarxiv}
\usepackage{graphicx}
\usepackage{booktabs}
\usepackage{amsmath,amssymb,amsfonts,amsthm}
\usepackage{microtype}
\usepackage{enumitem}
\usepackage{float}
\usepackage[utf8]{inputenc} 
\usepackage[T1]{fontenc}    
\usepackage{hyperref}       
\usepackage{url}            
\usepackage{booktabs}       
\usepackage{amsmath,amssymb,amsfonts}  
\usepackage{nicefrac}       
\usepackage{microtype}      
\usepackage{multirow}       
\usepackage{tabularx}       
\usepackage{xcolor}         
\usepackage{fancyhdr}       
\usepackage{graphicx}       
\usepackage[numbers]{natbib} 
\graphicspath{{./}{media/}} 
\newtheorem{theorem}{Theorem}[section]
\newtheorem{proposition}[theorem]{Proposition}
\newtheorem{lemma}[theorem]{Lemma}
\newtheorem{corollary}[theorem]{Corollary}
\theoremstyle{definition}

\newcommand{\R}{\mathbb{R}}

\newcommand{\norm}[1]{\left\lVert #1\right\rVert}

\newcommand{\fip}[2]{\left\langle #1,#2\right\rangle_F}
\newcommand{\tr}{\operatorname{tr}}
\newcommand{\rms}{\operatorname{RMS}}
\newcommand{\softmax}{\operatorname{softmax}}
\newcommand{\onehot}{\operatorname{onehot}}
\newcommand{\vecop}{\operatorname{vec}}

\allowdisplaybreaks
\title{LoCA: Forward-Only LLM Tuning after One-Shot Calibration with Local Credit Assignment
}

\author{
  Linhan Xia\textsuperscript{1,4,\textdagger}\quad
  Rui Liu\textsuperscript{2,\textdagger}\quad
  Zhaofeng Zhang\textsuperscript{3,5,\textdagger}\quad
  Yihao Wang\textsuperscript{6}\quad
  Binrui Shen\textsuperscript{7}\quad 
  Shengxin Zhu\textsuperscript{7,8}* \\
  \textsuperscript{1}University of Oklahoma \quad
  \textsuperscript{2}Imperial College London \quad
  \textsuperscript{3}University of Michigan \\
  \textsuperscript{4}Tencent \quad
  \textsuperscript{5}University of Edinburgh \quad
  \textsuperscript{6}University of Southern California \\
  \textsuperscript{7}Beijing Normal University \quad
  \textsuperscript{8}Beijing Normal-Hong Kong Baptist University \\
  \textsuperscript{\textdagger}These authors contributed equally to this work. 
  * Corresponding Author: \texttt{shengxin.zhu@bnu.edu.cn}
}

\begin{document}
\maketitle

\begin{abstract}
Parameter-efficient post-training reduces the number of trainable parameters, but still requires repeated end-to-end backpropagation through the frozen backbone. Every adaptation step therefore needs backward-capable hardware and must store or recompute activations. We ask whether this repeated backward chain can be replaced by a one-time calibration. We introduce Local Credit Assignment (LoCA), a two-stage method for small-shift adaptation. One probe backward pass fits a low-rank map at each transformer block from the final prediction error to a local hidden-state correction. LoCA then reuses these maps to form blockwise regression targets from forward activations and fits low-rank adapters with closed-form ridge solves. No further backbone backward pass is required. We evaluate LoCA on five discriminative benchmarks with Qwen2.5 models from 0.5B to 14B. In 16 of 25 reported task--scale comparisons, LoCA yields lower evaluation cross-entropy than the corresponding LoRA run. Its measured full-run GPU peak, including calibration, is 26--29\% lower than LoRA's. After calibration, its CPU steady-state memory is 36--52\% lower and its per-pass time is 43--48\% lower. A shared scale-normalized candidate set is reused across all tested Qwen2.5 sizes and on SmolLM2-1.7B. LoCA thus amortizes global credit assignment into one calibration and enables later forward-only tuning when repeated backpropagation is impractical. The code associated with this paper is available \href{https://github.com/Xia12121/LoCA}{here}.
\end{abstract}


\section{Introduction}\label{sec:Introduction}
Post-training includes instruction tuning \citep{ouyang2022training}, domain adaptation \citep{gururangan2020dont}, and preference alignment \cite{bai2022constitutional}. Even when only adapters are updated, these methods usually backpropagate through the frozen backbone. They must store or recompute activations, propagate gradients, and maintain optimizer states \citep{kingma2015adam}. Fine-tuning can therefore use up to $12\times$ the memory of inference \citep{malladi2023mezo}. The need for repeated backpropagation substantially increases the cost, making the method less practical than forward-only inference.

Parameter-efficient fine-tuning reduces trainable parameters and optimizer states, but it retains the backward chain. LoRA \citep{hu2022lora} and adapter- or prompt-based methods \citep{houlsby2019adapters,li2021prefix,lester2021prompt} still propagate gradients through the backbone. For example, PEFT for a 13B model can require $6\times$ inference memory \citep{malladi2023mezo}. These methods reduce several costs of backpropagation, but do not remove repeated backward execution.

Forward-only methods offer another trade-off. MeZO estimates update directions from function values \citep{spall1992spsa,malladi2023mezo}. Its memory use is close to inference, while its convergence and tuning can be affected by estimator variance \citep{duchi2015optimal,nesterov2017random}. Feedback alignment uses fixed feedback matrices \citep{lillicrap2016random,nokland2016dfa}, but often relies on forward weights adapting to those matrices \citep{refinetti2021align}. A frozen backbone cannot use this mechanism. Local learning avoids global propagation \citep{belilovsky2019greedy,nokland2019local,hinton2022forward}, but its targets usually support representation learning or compression \citep{jiao2020tinybert}. Taken together, this trade-off motivates us to combine their complementary strengths by using a proper scheduling.

We study small, targeted changes to a strong pretrained model, a more common practical setting than training a model from scratch. Our goal is to avoid the high computational cost in this regime. Prior work uses low-rank parameter updates for this setting \citep{li2018intrinsic,aghajanyan2021intrinsic,hu2022lora}. We apply the same local view to credit assignment. When the representation change is small, we separate two tasks. The first estimates how each block output should change. The second fits adapter parameters to that change.

We introduce \textbf{Local Credit Assignment fine-tuning (LoCA)} for this two-stage setting. During calibration, one probe backward pass fits a low-rank feedback operator for each block. During adaptation, these operators turn the final-layer error into local adapter targets. Each adapter is fitted by ridge regression. After calibration, an outer iteration needs forward execution and local linear solves, but no backbone backward pass. A device without backward support must receive operators calibrated for the same checkpoint and target distribution.

In the tested discriminative tasks, LoCA recovers most of LoRA's cross-entropy improvement and uses less measured memory. We also reuse one normalized candidate set across model sizes and on a second model family. Results from individual cells are treated as observations, not as evidence that LoCA is generally better than LoRA.

Our contributions are as follows.
\begin{enumerate}
  \item \textbf{Calibrated local credit.}
  LoCA reuses feedback operators from one probe backward pass during a later forward-only tuning stage.

  \item \textbf{Closed-form local updates.}
  Each residual adapter is fitted by a ridge solve. A scale-normalized target lets us reuse the same candidate set across the tested models.

  \item \textbf{Empirical study.}
  We compare quality, memory, and iteration count with LoRA and matched-adapter MeZO on five benchmarks and 0.5B--14B models.
\end{enumerate}

This paper is organized as follows. Section~\ref{sec:Related Works} reviews related work on efficient fine-tuning and local training methods. Section~\ref{sec:Methodology} presents the LoCA framework and its closed-form local update procedure. Section~\ref{sec:Experiment} evaluates LoCA against representative baselines and analyzes its efficiency and limitations. Section~\ref{sec:Discussion} discusses the scope of the method, and Section~\ref{sec:Conclusion} concludes the paper.

\section{Related Work}\label{sec:Related Works}
\textbf{Parameter- and memory-efficient fine-tuning.} Adapters \citep{houlsby2019adapters}, prompt methods \citep{li2021prefix,lester2021prompt}, and LoRA \citep{hu2022lora} reduce the number of trainable parameters. Other methods reduce memory in different ways. GaLore projects gradients \citep{zhao2024galore}, LOMO combines gradient computation and parameter updates \citep{lv2024lomo}, QLoRA quantizes frozen weights \citep{dettmers2023qlora}, and checkpointing recomputes activations \citep{chen2016checkpoint}. BAdam updates blocks in sequence \citep{luo2024badam}. These methods reduce parts of the training cost, but still use backward differentiation through at least part of the model. LoCA keeps a low-rank adapter but changes how its update is obtained.

\textbf{Zeroth-order and forward-only optimization.} Zeroth-order methods estimate update directions from function values \citep{spall1992spsa}. MeZO applies this idea to LLM fine-tuning with inference-level memory \citep{malladi2023mezo}. Recent work uses low-rank \citep{yu2025subzero,chen2025lozo}, sparse \citep{liu2024sparsemezo}, or curvature-aware perturbations \citep{zhao2025hizoo}. These methods retain the memory benefit of forward-only execution, but their iteration count and tuning can depend on estimator variance \citep{duchi2015optimal,nesterov2017random}. We compare with Adapter-MeZO, which perturbs the same low-rank adapter parameters as LoCA. LoCA uses fixed local solves rather than stochastic perturbation updates.

\textbf{Feedback alignment.} Feedback alignment trains networks without exact transposed weights \citep{lillicrap2016random}. Direct feedback alignment sends the top-layer error to each layer \citep{nokland2016dfa}, although transformers remain difficult \citep{launay2020dfa}. Its success often depends on forward weights adapting to the fixed feedback \citep{refinetti2021align}. This option is limited when the backbone is frozen. LoCA instead fits the feedback once to the frozen model and then keeps it fixed.

\textbf{Local and layer-wise learning.} Local learning uses separate objectives for different layers \citep{belilovsky2019greedy,nokland2019local,hinton2022forward}. Forward gradients have also been combined with local losses \citep{ren2023scaling}. Most of these targets support representation learning or compression \citep{jiao2020tinybert}. LoCA keeps the blockwise structure, but forms its targets from a first-order change in the post-training loss. It then fits each adapter with a ridge solve.

\section{Methodology}\label{sec:Methodology}
LoCA assumes that adaptation stays near the calibrated model state. It separates credit calibration from adapter fitting. Calibration uses one probe backward pass to fit low-rank maps from the top-layer error to blockwise credit signals. After calibration, each outer iteration uses no backbone backward pass. The block parameter is found by solving a local ridge objective. This solve is exact for a fixed local target. It does not imply that global cross-entropy decreases at every outer iteration. Figure~\ref{fig:lora_loca_overview} shows the post-calibration loop.

\begin{figure*}[t]
  \centering
  \includegraphics[width=0.80\textwidth]{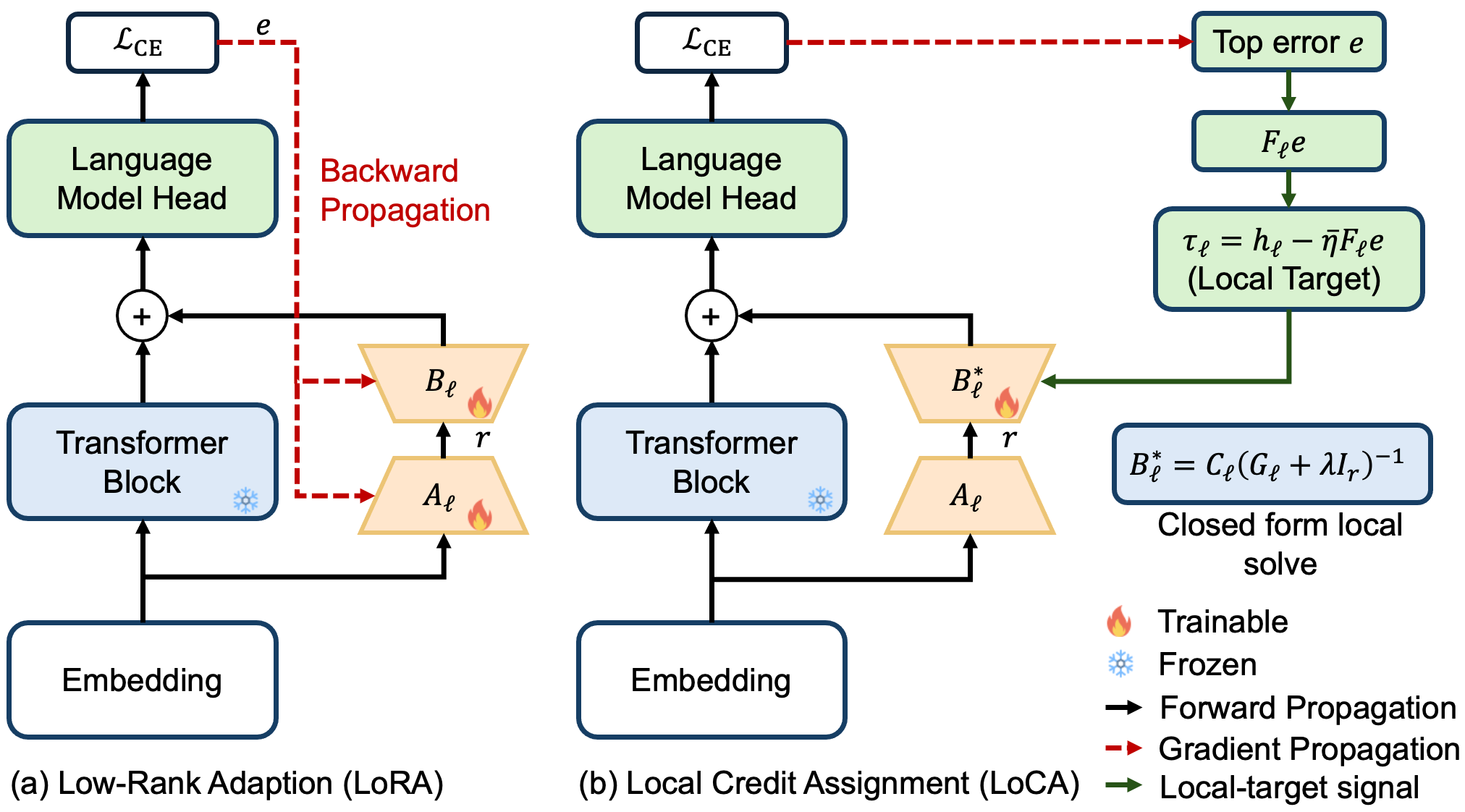}
  \caption{LoRA and the post-calibration LoCA loop. One probe backward pass fits $F_\ell$ before this loop. Later iterations use forward propagation and local ridge solves.}
  \label{fig:lora_loca_overview}
\end{figure*}

\subsection{Local Targets}\label{sec:targets}

Following the setting of LoRA \cite{hu2022lora}, we consider a pre-trained transformer with \(L\) frozen blocks. Let \(h_\ell \in \mathbb{R}^d\) be the hidden state after block \(\ell\), with the hidden embedding \(h_0\). A frozen head \(W_u \in \mathbb{R}^{V \times d}\) maps \(h_L\) to logits \(z=W_u h_L\) over vocabulary of size \(V\). We attach a low-rank adapter to each block. The adapted hidden state is
\begin{equation}
  h_\ell = \bar{h}_\ell + B_\ell A_\ell h_{\ell-1},
  \label{eq:adapter}
\end{equation}
where \(\bar{h}_\ell\) is the frozen block output, \(A_\ell \in \mathbb{R}^{r \times d}\) is a fixed random projection, and \(B_\ell\in \mathbb{R}^{d \times r}\) is fitted. Setting \(B_\ell=0\) recovers the frozen model. For one block solve, the current forward states and all other adapters are fixed. Thus, $\bar{h}_\ell$ and $p_\ell = A_\ell h_{\ell-1}$ are fixed with respect to $B_\ell$. The state \(h_\ell\) is then linear in \(B_\ell\), which makes the block objective a convex quadratic.

Unlike LoRA \cite{hu2022lora}, LoCA does not update every adapter from one global CE gradient. It gives each block a regression target \(\tau_\ell \in \mathbb{R}^d\) and solves
\begin{equation}
  \min_{B_\ell}\;\mathbb{E}_x \bigl\|\, h_\ell - \tau_\ell \,\bigr\|^2
  \label{eq:per-block}
\end{equation}
for each block. The post-calibration solve does not propagate gradients across blocks. The remaining question is how to choose $\tau_\ell$.

A perturbation \(\delta h_\ell\) changes the loss by \(\delta\mathcal{L} \approx \langle g_\ell,\, \delta h_\ell \rangle\), where \(g_\ell = \partial\mathcal{L}/\partial h_\ell\). A first-order target is
\begin{equation}
    \hat{\tau}_\ell = h_\ell - \eta\, g_\ell.
    \label{eq:target}
\end{equation}
In the conventional setting, computing \(g_\ell\) at every optimization step requires the repeated backward passes we aim to eliminate. We approximate it during tuning by replacing the input-dependent Jacobian product with a fixed matrix $F_\ell \in \mathbb{R}^{d \times d}$:
\begin{equation}
  \tau_\ell \;=\; h_\ell \;-\; \eta\, F_\ell\, e,
  \label{eq:target1}
\end{equation}
where \(e=W_u^{\top}\bigl(\mathrm{softmax}(z) - \mathrm{onehot}(y)\bigr)\) is the top-layer error from the frozen head. Once $F_\ell$ is fixed, this target does not require a backbone backward pass.

\subsection{The Feedback Operator}\label{sec:feedback}
The operator $F_\ell$ maps the top-layer error $e$ to an estimate of $g_\ell$. Random feedback gives a mean cosine alignment near zero in our calibration test, $\alpha_\ell = \cos\angle(F_\ell e, g_\ell) \approx -0.005$. It therefore provides little useful direction in this setting. Feedback alignment can improve as forward weights change \citep{refinetti2021align}, but the LoCA backbone is frozen. We instead fit $F_\ell$ to the frozen model.

We fit an untruncated feedback map $\widehat F_\ell$ once from samples of the backward mapping. One probe backward pass through the frozen model yields per-token gradients $g_\ell$ and top-layer errors $e$, stacked as $\widetilde{G}_\ell$ and $\widetilde{E}$. We solve
\begin{equation}
  \widehat F_\ell
  =
  \arg\min_{F}\;
  \big\| F\widetilde{E}-\widetilde{G}_\ell \big\|_F^2
  +\beta\|F\|_F^2,
  \label{eq:sketch}
\end{equation}
where $\beta>0$. Writing its singular value decomposition as
$\widehat F_\ell=U\Sigma V^\top$, we retain
$F_\ell=U_k\Sigma_kV_k^\top$ with $k=8$.
Thus, $F_\ell$ is a rank-$k$ approximation to the regularized fit rather than the full fitted map. On the probe data, its cosine alignment with the true hidden-state gradients is $0.3$--$0.5$. We treat this fixed map as a local approximation whose alignment may degrade as repeated updates move the model away from the calibrated state. In the small-shift regime, however, it can carry useful backward-derived credit into forward-only adaptation without requiring further backward passes.

\subsection{Closed-Form Per-Block Solve}\label{sec:solve}
For a fixed \(\tau_\ell\), each block gives a ridge problem. Its residual target is
\begin{equation}
    \rho_\ell := \tau_\ell - \bar{h}_\ell = B_\ell^{(t)} p_\ell - \eta\, F_\ell e,
    \label{eq:residual}
\end{equation}
where $B_\ell^{(t)}$ is the current adapter. Over $N$ predicted tokens, block $\ell$ minimizes
\begin{equation}
    \sum_n \| B_\ell p_\ell^{(n)} - \rho_\ell^{(n)} \|^2 + \lambda \|B_\ell \|_F^2.
\end{equation}
This strictly convex problem has the unique closed-form minimizer
\begin{equation}
  B_\ell^{\star} = C_\ell \big( G_\ell + \lambda I_r \big)^{-1},
  \label{eq:solution}
\end{equation}
where \(\lambda>0\) is the ridge regularization coefficient, and
\begin{equation}
  G_\ell = \sum_n p_\ell^{(n)} p_\ell^{(n)\top},
  \qquad
  C_\ell = \sum_n \rho_\ell^{(n)} p_\ell^{(n)\top}.
  \label{eq:gram}
\end{equation}
The block parameter uses no gradient optimizer, learning-rate schedule, or optimizer state. The coefficient $\eta$ remains a target-size hyperparameter.

Three details are useful in practice. First, $G_\ell$ and $C_\ell$ are additive over batches. Their persistent storage is $O(L(r^2+dr))$ and does not grow with the number of training tokens. Calibration creates one temporary backward graph. Its peak is included in the GPU measurements in Table~\ref{tab:resource}. Second, we keep the frozen model as a candidate and return it when no adapted snapshot improves held-out CE. This fallback protects the selected validation objective, not unseen test metrics. Third, the block solve is deterministic when the probe data, projection, data order, and numerical implementation are fixed. These choices can still cause variation across runs.

\begin{table*}[!t]
\centering
\setlength{\tabcolsep}{4.5pt}

\begin{tabular*}{\textwidth}{@{\extracolsep{\fill}}llcccccc@{}}
\toprule
\textbf{Model} & \textbf{Task}
& \textbf{Frozen}
& \textbf{LoRA}
& \textbf{Adapter-MeZO}
& \textbf{LoCA}
& $\boldsymbol{R_{\mathrm{CE}}}$
& $\boldsymbol{R_{\mathrm{acc}}}$ \\
\midrule

\multirow{5}{*}{0.5B}
     & SST-2      & 2.463 & 0.261 & \textbf{0.183} & 0.399          & 0.94          & 0.96 \\
     & BoolQ      & 0.827 & 0.409 & \textbf{0.315}         & 0.363 & \textbf{1.11} & 0.82 \\
     & ARC-Easy   & 2.990 & 2.210 & 2.319          & \textbf{2.123} & \textbf{1.11} & 0.79 \\
     & OpenBookQA & 4.707 & 4.264 & 3.837          & \textbf{3.645} & \textbf{2.40} & 0.85 \\
     & HellaSwag  & 3.450 & 2.920 & 3.050          & \textbf{2.860} & \textbf{1.11} & 0.72 \\

\midrule

\multirow{5}{*}{1.5B}
     & SST-2      & 3.894 & \textbf{0.112} & 0.204 & 0.448          & 0.91          & 0.87 \\
     & BoolQ      & 0.527 & 0.399          & 0.276 & \textbf{0.304} & \textbf{1.74} & 0.82 \\
     & ARC-Easy   & 2.398 & 2.156
                          & 2.065
                          & \textbf{1.851}
                          & \textbf{2.26} & 0.77 \\
     & OpenBookQA & 4.120 & 3.700          & 3.910 & \textbf{3.620} & \textbf{1.19} & 0.74 \\
     & HellaSwag  & 2.917 & 2.655          & 2.717 & \textbf{2.635} & \textbf{1.08} & 0.82   \\

\midrule

\multirow{5}{*}{3B}
     & SST-2      & 8.294 & \textbf{0.184} & 0.213 & 0.283          & 0.98          & 0.96 \\
     & BoolQ      & 1.914 & 0.325          & 0.287 & \textbf{0.269} & \textbf{1.04} & 0.85 \\
     & ARC-Easy   & 2.621 & 2.147          & 1.984 & \textbf{1.859} & \textbf{1.61} & \textbf{1.31} \\
     & OpenBookQA & 3.980 & \textbf{3.550} & 3.900 & 3.610          & 0.86          & 0.69 \\
     & HellaSwag  & 2.740 & \textbf{2.480} & 2.690 & 2.520          & 0.85          & 0.71 \\

\midrule

\multirow{5}{*}{7B}
     & SST-2      & 5.200 & \textbf{0.160} & 0.420          & 0.205          & 0.99          & 0.93 \\
     & BoolQ      & 0.819 & 0.426 & 0.293          & \textbf{0.253} & \textbf{1.44} & $\ddagger$ \\
     & ARC-Easy   & 2.442 & 2.086 & 1.754          & \textbf{1.721} & \textbf{2.03} & 0.87 \\
     & OpenBookQA & 3.860 & 3.380 & 3.790          & \textbf{3.320} & \textbf{1.13} & 0.78 \\
     & HellaSwag  & 2.610 & \textbf{2.340} & 2.580          & 2.390          & 0.81          & 0.84 \\

\midrule

\multirow{5}{*}{14B}
     & SST-2      & 4.850 & 0.135          & 0.510 & \textbf{0.128} & \textbf{1.00} & 0.97 \\
     & BoolQ      & 1.445 & 0.426          & \textbf{0.183} & 0.221 & \textbf{1.20} & $\ddagger$ \\
     & ARC-Easy   & 2.444 & 2.438$^\dagger$ & \textbf{1.474} & 1.642 & $\dagger$      & 0.95 \\
     & OpenBookQA & 3.740 & \textbf{3.220} & 3.760 & 3.290          & 0.87          & 0.70 \\
     & HellaSwag  & 2.500 & \textbf{2.200} & 2.520 & 2.240          & 0.87          & 0.88 \\
\bottomrule
\end{tabular*}

\begin{flushleft}
\footnotesize
$^\dagger$ indicates that the corresponding recovery ratio is not meaningful because the LoRA improvement over the frozen model is negligible.
$\ddagger$ indicates that the accuracy recovery is omitted because the LoRA accuracy gain is zero or near zero.
\end{flushleft}

\caption{Main benchmark results. We report evaluation CE, where lower is better, together with CE and accuracy recovery ratios.}
\label{tab:quality}
\end{table*}

\subsection{Full Algorithm}
\label{sec:outer}
The closed-form solution above gives $B_\ell^\star$ for a fixed target. The target changes after an adapter update because later hidden states and the error $e$ also change. LoCA therefore repeats the procedure for $T$ outer steps. Each step runs a forward pass, updates the targets, and solves the block objectives. A Jacobi schedule uses one shared forward pass for all blocks. A Gauss--Seidel schedule refreshes later states after each block. Each solve minimizes its current local objective. Global CE can still be non-monotone, so we select $T$ by held-out CE.

The step size \(\eta\) in~\eqref{eq:target1} is an absolute quantity, and the norm of \(e\) varies across model sizes. In practice, a fixed \(\eta\) that works at 0.5B returns the frozen model at 14B. We reduce this sensitivity by normalizing the target correction to the residual-stream scale:
\begin{equation}
  \tau_\ell = h_\ell - \bar{\eta}\, F_\ell\, e,
  \qquad
  \bar{\eta} = \eta \cdot
  \frac{\mathrm{RMS}(h_{\ell-1})}{\mathrm{RMS}(F_\ell\, e)},
  \label{eq:scalenorm}
\end{equation}
so $\eta$ is the relative change to the residual stream. We reuse the same candidate set across all tested model sizes and architectures.

\section{Experiments}\label{sec:Experiment}
\subsection{Setup}\label{sec:exp-setup}
\textbf{Models and tasks.} We test Qwen2.5 at 0.5B, 1.5B, 3B, 7B, and 14B parameters \citep{yang2024qwen25}. We also test SmolLM2-1.7B \citep{allal2025smollm2}. The tasks are SST-2 \citep{socher2013recursive}, BoolQ \citep{clark2019boolq}, ARC-Easy and ARC-Challenge \citep{clark2018arc}, OpenBookQA \citep{mihaylov2018obqa}, and HellaSwag \citep{zellers2019hellaswag}.

\textbf{Baselines and metrics.} We compare with the frozen model, LoRA \citep{hu2022lora}, and MeZO \citep{malladi2023mezo}. Full-parameter MeZO perturbs a much larger parameter space and weakens with model size in supplementary Figure~S1. Our main comparison therefore uses Adapter-MeZO, which perturbs the same low-rank adapter parameters as LoCA and LoRA. We report per-token CE and ranking accuracy. Recovery gives a normalized summary of the frozen-to-LoRA gap.

Table~\ref{tab:quality} reports the results.

\begin{equation}
  R_{\mathrm{CE}}
  =
  \frac{
    \mathrm{CE}_{\mathrm{frozen}} - \mathrm{CE}_{\mathrm{ours}}
  }{
    \mathrm{CE}_{\mathrm{frozen}} - \mathrm{CE}_{\mathrm{LoRA}}
  },
  \label{eq:recovery-ce}
\end{equation}
\begin{equation}
  R_{\mathrm{acc}}
  =
  \frac{
    \mathrm{acc}_{\mathrm{ours}} - \mathrm{acc}_{\mathrm{frozen}}
  }{
    \mathrm{acc}_{\mathrm{LoRA}} - \mathrm{acc}_{\mathrm{frozen}}
  },
  \label{eq:recovery-acc}
\end{equation}
so $R=1$ matches the reported LoRA value. Recovery can be unstable when the frozen-to-LoRA gap is near zero. We therefore treat raw CE and accuracy as the main metrics and use recovery as a descriptive summary. We omit ratio claims for near-zero denominators.

\textbf{Protocol.} LoCA uses rank $r=32$, at most 40 outer iterations, and early stopping on held-out CE. We select $\eta$ from the stated candidate set on the same held-out metric. MeZO scales its learning rate by $1/\sqrt{D}$ and searches it at each model size. Unless noted otherwise, table cells are single runs with seed 0. We interpret cell-level differences as descriptive rather than statistically significant. Runs use NVIDIA L40S GPUs or 32 GB CPU-only servers. We do not include GaLore or LOMO because they retain backward differentiation and address a different cost trade-off.

\setcounter{topnumber}{1}
\begin{table}[!t]
\centering
\small
\setlength{\tabcolsep}{4.5pt}
\begin{tabular*}{\columnwidth}{@{\extracolsep{\fill}}lcccc@{}}
\toprule
 & Adapter-MeZO & \textbf{LoCA} & LoRA & LoCA/LoRA \\
\midrule
\multicolumn{5}{@{}l}{\emph{GPU peak training memory (MB)}} \\
0.5B & \textbf{2{,}197}  & 4{,}420  & 6{,}254  & $0.71\times$ \\
1.5B & \textbf{5{,}669}  & 6{,}646  & 9{,}109  & $0.73\times$ \\
3B   & \textbf{9{,}079}  & 9{,}248  & 12{,}612 & $0.73\times$ \\
7B   & \textbf{9{,}154}  & 11{,}520 & 15{,}647 & $0.74\times$ \\
\midrule
\multicolumn{5}{@{}l}{\emph{CPU steady-state memory (MB)}} \\
0.5B & \textbf{879}  & 1{,}279 & 2{,}664 & $0.48\times$ \\
1.5B & \textbf{1{,}741} & 2{,}845 & 5{,}633 & $0.51\times$ \\
3B   & \textbf{2{,}015} & 2{,}590 & 4{,}875 & $0.53\times$ \\
7B   & \textbf{3{,}533} & 4{,}186 & 6{,}582 & $0.64\times$ \\
\midrule
\multicolumn{5}{@{}l}{\emph{CPU wall-clock per pass (s)}} \\
0.5B & 8.9  & \textbf{5.9}  & 10.3 & $0.57\times$ \\
1.5B & 26.0 & \textbf{14.2} & 25.5 & $0.56\times$ \\
3B   & 28.7 & \textbf{15.0} & 29.0 & $0.52\times$ \\
7B   & 55.4 & \textbf{29.5} & 55.2 & $0.53\times$ \\
\bottomrule
\end{tabular*}
\caption{Resources beyond the loaded model. GPU peak includes probe calibration. CPU steady-state memory and per-pass time describe post-calibration tuning.}
\label{tab:resource}
\end{table}
\subsection{Main Results}\label{sec:exp-quality}

Experiment~1 compares the frozen model, LoRA, Adapter-MeZO, and LoCA on five model sizes and five tasks.

LoCA's CE recovery ranges from 0.81 to 2.40 in the reported cells. It obtains lower CE than the corresponding LoRA run in 16 of the 25 task--scale cells. These cells are single runs, so the count is descriptive and does not show general superiority over LoRA.

On BoolQ at 7B and 14B, the reported LoRA runs lower CE while ranking accuracy falls from 0.824 to 0.636 and from 0.852 to 0.364. LoCA reaches 0.840 and 0.864 in these cells. This CE--accuracy gap is consistent with completion-format overfitting, but the current experiment does not isolate that cause.

\begin{table}[t]
\centering
\small
\setlength{\tabcolsep}{4.0pt}
\begin{tabular*}{\columnwidth}{@{\extracolsep{\fill}}llccc@{}}
\toprule
Model & Task & abs-$\eta$ $R_{\mathrm{CE}}$ & rms CE & rms $R_{\mathrm{CE}}$ \\
\midrule
0.5B & SST-2 & 0.94 (tuned) & 0.249 & \textbf{1.02} \\
1.5B & SST-2 & 0.91 (tuned) & 0.168 & \textbf{1.00} \\
3B & SST-2 & 0.89 (tuned) & 0.148 & \textbf{0.92} \\
7B   & SST-2 & 0.82 (tuned) & 0.152 & \textbf{0.94} \\
\midrule
SmolLM2 & SST-2 & ${\approx}0$ & 0.371 & \textbf{0.93} \\
SmolLM2 & BoolQ & 0.06 & 0.391 & \textbf{1.02} \\
SmolLM2 & ARC-Easy & 0.07 & 1.677 & \textbf{1.29} \\
SmolLM2 & ARC-Chal. & 0.02 & 2.318 & \textbf{0.93} \\
\bottomrule
\end{tabular*}
\caption{Scale-normalized targets with one shared candidate set across the tested models.}
\label{tab:scalenorm}
\end{table}

\subsection{Resource Use}\label{sec:exp-resource}
Table~\ref{tab:resource} reports memory beyond the loaded model. LoCA uses $0.71$--$0.74\times$ the GPU peak memory of LoRA and $0.48$--$0.64\times$ its CPU steady-state memory. The GPU peak is measured over the full LoCA run and includes the temporary graph used for calibration. After calibration, LoCA stores the streaming statistics $G_\ell\in\mathbb{R}^{r\times r}$ and $C_\ell\in\mathbb{R}^{d\times r}$. Their size is $O(L(r^2+dr))$ and does not grow with the number of training tokens. On BoolQ at length 512, LoRA reaches 32.3\,GB and LoCA reaches 16.3\,GB. LoRA is killed on the 32\,GB device, while LoCA completes the run.

Adapter-MeZO uses the least memory in the table. It perturbs parameters in place and stores no gradient graph or optimizer state. Its update estimate is stochastic and needs $10^3$--$10^4$ perturbation steps in our runs, together with a learning-rate search at each scale.

After calibration, one LoCA pass costs about one forward pass. On CPU, its measured per-pass time is $1.7$--$1.9\times$ lower than LoRA's. On 1.5B ARC-Easy, 3{,}000 Adapter-MeZO steps take 4.2 hours. On 1.5B ARC-Challenge, the measured LoCA run reaches lower CE in $1.7\times$ less wall-clock. We did not record calibration latency separately, so the general timing claim is limited to post-calibration cost. The GPU peak-memory result still includes calibration.

\subsection{Cross-Family Results}\label{sec:exp-scalenorm}
The target-size coefficient $\eta$ is sensitive to model scale. We normalize it by the residual-stream norm, which makes it a relative change. We use the candidate set $\{0.003,0.01,0.03\}$ at every tested scale and select a value by held-out CE. The search range is not redesigned for each model.

On SST-2, the normalized candidate set gives recoveries of 1.02, 1.00, 0.92, and 0.94 across four sizes. The hand-tuned absolute values give 0.94, 0.91, 0.89, and 0.82. Thus, one shared candidate set works across the tested scales, although held-out selection is still used within that set.

We also test SmolLM2-1.7B. Reusing the absolute $\eta$ gives recovery from 0 to 0.07 across four tasks. With target normalization, recovery ranges from 0.93 to 1.29 and averages about 1.04. These results support the use of relative target size on a second model family. They do not show that normalization will transfer to every architecture.

\section{Discussion}\label{sec:Discussion}

Across the 25 combinations of Qwen2.5 model scales and tasks, LoCA reports lower evaluation CE than the corresponding LoRA run in 16 cases. Although these cells are mainly single runs, and therefore do not establish that LoCA is generally better than LoRA, they show that LoCA is not merely an inaccurate copy of the LoRA update. Whereas LoRA follows the global CE gradient, LoCA optimizes regularized low-rank objectives and selects a snapshot by held-out CE. These choices define a different adaptation path. Ridge regularization, the low-rank constraint, and held-out selection may jointly discourage large local changes, although our experiments do not isolate their effects. BoolQ illustrates another difference. At 7B and 14B, LoRA lowers CE while ranking accuracy falls from 0.824 to 0.636 and from 0.852 to 0.364. LoCA reaches accuracies of 0.840 and 0.864 in the same cells. While this pattern is consistent with completion-format overfitting, the current experiments do not establish that cause. More broadly, because the local objectives impose a different inductive bias from global gradient updates, they may also lead to different generalization behavior.

Although both LoCA and Adapter-MeZO reduce the need for repeated backward passes, they obtain update information in different ways. MeZO applies zeroth-order SGD to language-model tuning, estimating an update direction from random perturbations and function-value differences \citep{malladi2023mezo}. This gives MeZO an inference-level memory footprint, but each query provides a stochastic directional estimate rather than an exact solution to the current objective. Classical zeroth-order theory shows that, when gradients are replaced by random function-value estimators, the resulting methods incur additional variance and may suffer dimension-dependent query or convergence costs \citep{duchi2015optimal,nesterov2017random}. It's also worth mentioning that, although random feedback is distinct from zeroth-order estimation, the random-map tested when constructing feedback operator exhibits near-zero in-probe alignment with the true hidden-state gradient, whereas calibration produces substantially better-aligned feedback. These analyses rely on assumptions that do not directly describe non-convex LLM tuning. Nevertheless, they identify a general cost of random function-value probes. Recent work has therefore introduced variance-reduced MeZO variants, which are designed to improve stability and convergence in language-model tuning \citep{gautam2024variance}.

By contrast, once the forward states and local target are fixed, the ridge objective used by LoCA is a regularized convex quadratic. When $\lambda>0$, this objective has a unique closed-form solution. The block update uses sufficient statistics accumulated from the current data, rather than repeatedly estimating random directions. It also avoids an iterative learning-rate schedule inside the block solve. The benefit of the closed form is not a guarantee of global optimality. Instead, it removes optimization error within the current local subproblem. Because other adapters and hidden states change after an update, the local target also changes across outer iterations. Global CE therefore need not decrease at every step.

Since the two methods pay for update information in different ways, their iteration counts also differ substantially in our experiments. Adapter-MeZO requires between $10^3$ and $10^4$ perturbation steps, whereas LoCA uses at most 40 outer iterations. LoCA still requires the target-size coefficient and ridge regularization to be selected, and its calibration cost is not zero. Moreover, because calibration latency was not recorded separately, the current measurements do not support a universal end-to-end speed claim. The observed trade-off is more specific. Adapter-MeZO has the lowest memory use, while LoCA exchanges one probe backward pass for structured local updates and fewer later iterations.

When the local correction is expressed relative to the residual-stream scale, the same candidate set can be reused across Qwen2.5 models from 0.5B to 14B and on SmolLM2-1.7B. An absolute correction, in contrast, changes substantially with model size and architecture. This result suggests that relative representation change is more portable than an absolute hidden-state correction. It does not make LoCA hyperparameter-free, since a value is still selected from the shared set by held-out CE. The supported claim is therefore that the candidate range transfers across the tested models, not that one fixed value works without validation.

Because the fixed feedback maps are calibrated near the base model, the present results are limited to discriminative tasks that require relatively small representation changes. The approximation may weaken when tuning moves far from the calibration point or when the task requires a larger shift. Such settings may require a new calibration pass or an input-dependent feedback map. Long-form generation and large distribution shifts are also outside the current study. LoCA should therefore be viewed as a two-stage method rather than a method that removes backward computation entirely. One probe backward pass calibrates the feedback maps, after which adaptation uses forward execution and local closed-form solves. A device without backward support can receive the calibrated maps from a training-capable host. The reported GPU peak includes calibration, whereas CPU steady-state memory and per-pass time describe the later stage. In this sense, LoCA amortizes repeated global credit assignment into one calibration pass, allowing subsequent tuning to proceed through forward computation and local solves.

\section{Conclusion}\label{sec:Conclusion}
We presented LoCA, a two-stage method that separates global credit calibration from adapter fitting. One probe backward pass fits low-rank maps from the final-layer error to blockwise corrections. After calibration, these maps turn forward activations into local targets, and each adapter is fitted by a closed-form ridge solve. The main idea is therefore to reuse calibrated credit information rather than reconstruct it through every tuning step.

Across five discriminative benchmarks and Qwen2.5 models from 0.5B to 14B, LoCA reports lower evaluation CE than the corresponding LoRA run in 16 of 25 comparisons. Since most cells are single runs, this count is descriptive and does not establish general superiority. The measured GPU peak over the full run, including calibration, is 26 to 29\% lower than LoRA's. After calibration, CPU steady-state memory is 36 to 52\% lower and per-pass time is 43 to 48\% lower. A shared scale-normalized candidate set is also reused across the tested Qwen2.5 sizes and on SmolLM2-1.7B, with held-out selection within that set.

The current evidence is limited to small-shift adaptation on discriminative tasks. Fixed feedback maps may become less accurate after a large model change or a long tuning path, and long-form generation remains untested. Future work should study when recalibration is needed and whether input-dependent feedback can support larger shifts. LoCA is not entirely free of backward computation. Instead, it defines a practical boundary: a training-capable host performs one calibration pass, while later adaptation can run through forward computation and local solves on an inference-oriented device.

\section*{Acknowledgments}
This work was supported in part by the National Key Technologies Research and Development Program (2025YFG0202100; 2025YFG0202600). The authors have no competing interests to declare that are relevant to the content of this article.

\bibliographystyle{unsrtnat}
\bibliography{refs}

\clearpage
\section*{Appendix}
\appendix

\section{Scope and Roadmap}

This section expands the opening of the main paper's Methodology section. It fixes the scope of the two-stage claim and states how the remaining sections support the equations in the main paper.

LoCA separates credit calibration from adapter fitting. In Stage~I, one probe backward pass through the frozen base model produces hidden-state gradients used to fit the feedback maps $\{F_\ell\}_{\ell=1}^{L}$. In Stage~II, these maps are fixed. Each outer iteration then uses a forward pass, the analytic error at the frozen language-model head, and blockwise ridge solves; it does not differentiate through a backbone block.

Sections~S2--S8 expand the Methodology section of the main paper. The mathematical statements have three distinct scopes:
\begin{enumerate}[leftmargin=2em]
 \item the language-model-head error and the two ridge minimizers are exact for the stated fixed data;
 \item fixed feedback defines a virtual hidden-state descent step only under explicit alignment and smoothness conditions;
 \item the fitted adapter need not realize that virtual step exactly, and the full outer iteration is not claimed to decrease global cross-entropy monotonically.
\end{enumerate}
Sections~S9--S13 provide experimental context for the setup, main benchmark, resource study, scale normalization, and MeZO baselines.

\section{Sequence-Level Notation and the Affine Adapter}

This section expands ``Local Targets'' and Eq.~(1) of the main paper. Its purpose is to show why a block solve is linear in $B_\ell$ while retaining the cross-token dependence created by self-attention.

Consider a decoder-only transformer with $L$ frozen blocks and hidden width $d$. For a sequence with $K$ positions, stack the hidden states after block $\ell$ by columns:
\begin{equation}
 H_\ell=[h_{\ell,1},\ldots,h_{\ell,K}]\in\R^{d\times K}.
\end{equation}
The vector $h_\ell$ in the main paper denotes one column of $H_\ell$. LoCA adds a direct low-rank correction to the output residual stream:
\begin{equation}
 H_\ell
 =\bar H_\ell+B_\ell A_\ell H_{\ell-1}
 =\bar H_\ell+B_\ell P_\ell^{\mathrm{all}},
 \label{eq:supp-adapter}
\end{equation}
where
\begin{equation}
 \bar H_\ell:=f_\ell^{\mathrm{base}}(H_{\ell-1}),\qquad
 A_\ell\in\R^{r\times d},\qquad
 B_\ell\in\R^{d\times r},\qquad
 P_\ell^{\mathrm{all}}:=A_\ell H_{\ell-1}\in\R^{r\times K}.
 \label{eq:supp-feature}
\end{equation}
For one block solve, the current forward states and all other adapters are fixed. Hence $\bar H_\ell$ and $P_\ell^{\mathrm{all}}$ are fixed with respect to $B_\ell$, and
\begin{equation}
 H_\ell=\bar H_\ell+B_\ell P_\ell^{\mathrm{all}}
 \label{eq:supp-linear-in-B}
\end{equation}
is exactly affine in $B_\ell$. The block solve therefore does not linearize the adapter.

Let $w_t\geq0$ be the loss weight at position $t$; a masked position has $w_t=0$. For one sequence, the scalar training loss is
\begin{equation}
 \mathcal L
 =\sum_{t=1}^{K}w_t\,\ell_{\mathrm{CE}}(h_{L,t},y_t).
 \label{eq:supp-sequence-loss}
\end{equation}
The ridge objective in the main paper is accumulated over $N$ selected prediction positions from the training data. Their projected features are denoted
\begin{equation}
 P_\ell=[p_\ell^{(1)},\ldots,p_\ell^{(N)}]\in\R^{r\times N}.
 \label{eq:supp-regression-features}
\end{equation}
Using only selected prediction positions is part of the local surrogate, not an exact reduction of the global gradient. A masked context position can still have a nonzero lower-layer gradient because it can affect later unmasked losses.

\section{Exact Head Error and Full Hidden-State Credit}

This section expands the top-layer error in ``Local Targets'' and clarifies what the feedback operator in ``The Feedback Operator'' approximates. It separates the exact head derivative from the cross-token, cross-layer gradient supplied by the probe backward pass.

The frozen language-model head is $W_u\in\R^{V\times d}$. At position $t$,
\begin{equation}
 z_t=W_uh_{L,t},\qquad
 \ell_{\mathrm{CE}}(h_{L,t},y_t)
 =-z_{t,y_t}+\log\sum_{j=1}^{V}\exp(z_{t,j}).
\end{equation}

\begin{lemma}[Top-layer error]
The contribution of position $t$ to the derivative with respect to the last hidden state is
\begin{equation}
 e_t
 :=w_t\nabla_{h_{L,t}}\ell_{\mathrm{CE}}
 =w_tW_u^\top\bigl(\softmax(z_t)-\onehot(y_t)\bigr).
 \label{eq:supp-top-error}
\end{equation}
Thus $e_t=0$ at a zero-weight position.
\end{lemma}

The vector $e$ in the main paper is the single-position case with unit weight. Here $w_t$ also covers masking and loss reduction.

\begin{proof}
For logit coordinate $i$,
\begin{equation}
 \frac{\partial\ell_{\mathrm{CE}}}{\partial z_{t,i}}
 =-\mathbf 1[i=y_t]
 +\frac{\exp(z_{t,i})}{\sum_j\exp(z_{t,j})}.
\end{equation}
Multiplication by $w_t$ and the chain rule through $z_t=W_uh_{L,t}$ give~\eqref{eq:supp-top-error}.
\end{proof}

Stack $E=[e_1,\ldots,e_K]\in\R^{d\times K}$ and define the exact gradient of the scalar sequence loss with respect to the block output by
\begin{equation}
 \mathcal G_\ell:=\nabla_{H_\ell}\mathcal L
 =[g_{\ell,1},\ldots,g_{\ell,K}]\in\R^{d\times K}.
\end{equation}
If $\mathcal J_{\ell+1:L}$ is the Jacobian of the current downstream adapted network, with all current adapters held fixed, then
\begin{equation}
 \vecop(\mathcal G_\ell)
 =\mathcal J_{\ell+1:L}^{\top}\vecop(E).
 \label{eq:supp-full-credit}
\end{equation}
Equivalently, causal attention gives the positionwise expression
\begin{equation}
 g_{\ell,s}
 =\sum_{t=s}^{K}
 \left(\frac{\partial h_{L,t}}{\partial h_{\ell,s}}\right)^\top e_t.
 \label{eq:supp-cross-token-credit}
\end{equation}
Consequently, $e_s=0$ does not generally imply $g_{\ell,s}=0$.

One probe backward pass computes all matrices $\{\mathcal G_\ell\}$ in~\eqref{eq:supp-full-credit}. During Stage~II, LoCA instead applies the same $d\times d$ map independently to the top errors:
\begin{equation}
 Q_\ell:=F_\ell E.
 \label{eq:supp-tokenwise-feedback}
\end{equation}
This is a tokenwise predictor of hidden-state credit, not an exact factorization of the full sequence Jacobian. The approximation begins at~\eqref{eq:supp-tokenwise-feedback}, not at the language-model head.

\section{Ideal Targets and Conditional Virtual Descent}

This section expands Eqs.~(3)--(4) in ``Local Targets'' and the alignment discussion in ``The Feedback Operator.'' It explains why the ideal target matches the global gradient locally and gives the precise condition under which fixed feedback defines a virtual descent step.

For a fixed current forward pass, the differential of the global loss caused by a perturbation $\Delta H_\ell$ is
\begin{equation}
 d\mathcal L
 =\fip{\mathcal G_\ell}{\Delta H_\ell}.
\end{equation}
Since $dH_\ell=(dB_\ell)P_\ell^{\mathrm{all}}$, the exact global adapter gradient is
\begin{equation}
 \nabla_{B_\ell}\mathcal L
 =\mathcal G_\ell(P_\ell^{\mathrm{all}})^\top.
 \label{eq:supp-global-adapter-gradient}
\end{equation}
If $\mathcal G_\ell$ were available at every iteration, the ideal hidden-state target with absolute representation-space step $s_\ell>0$ would be
\begin{equation}
 \widehat H_\ell=H_\ell-s_\ell\mathcal G_\ell.
 \label{eq:supp-ideal-target}
\end{equation}
At the current $B_\ell$, the gradient of
$\frac12\norm{\bar H_\ell+B_\ell P_\ell^{\mathrm{all}}-\widehat H_\ell}_F^2$
equals $s_\ell\nabla_{B_\ell}\mathcal L$. Thus the ideal full-state target reproduces the global adapter-gradient direction at the current forward state.

LoCA replaces $\mathcal G_\ell$ by $Q_\ell=F_\ell E$ and forms the virtual target
\begin{equation}
 T_\ell=H_\ell-s_\ell Q_\ell.
 \label{eq:supp-feedback-target}
\end{equation}
The following result concerns this virtual change in one block output. It does not yet concern the fitted adapter.

\begin{proposition}[Single-block virtual descent under alignment]
Consider one block and one batch, so $s_\ell$ is a scalar. Fix the current downstream adapted network and suppose that the upper-network loss, viewed as a function of $H_\ell$, has $M_\ell$-Lipschitz gradient on a convex set containing
$\{H_\ell-a s_\ell Q_\ell:0\leq a\leq1\}$, where $M_\ell>0$. Then
\begin{equation}
 \mathcal L(H_\ell-s_\ell Q_\ell)-\mathcal L(H_\ell)
 \leq
 -s_\ell\fip{\mathcal G_\ell}{Q_\ell}
 +\frac{M_\ell s_\ell^2}{2}\norm{Q_\ell}_F^2.
 \label{eq:supp-descent-lemma}
\end{equation}
If $Q_\ell\neq0$ and
\begin{equation}
 \fip{\mathcal G_\ell}{Q_\ell}
 \geq\gamma_\ell\norm{Q_\ell}_F^2
 \qquad(\gamma_\ell>0),
 \label{eq:supp-alignment-margin}
\end{equation}
then every
\begin{equation}
 0<s_\ell<\frac{2\gamma_\ell}{M_\ell}
 \label{eq:supp-step-bound}
\end{equation}
gives a strict decrease for the virtual single-block step.
\end{proposition}

\begin{proof}
The descent lemma on the stated convex set gives
\begin{equation}
 \mathcal L(H_\ell+\Delta)
 \leq\mathcal L(H_\ell)
 +\fip{\mathcal G_\ell}{\Delta}
 +\frac{M_\ell}{2}\norm{\Delta}_F^2.
\end{equation}
Substitute $\Delta=-s_\ell Q_\ell$ and then use~\eqref{eq:supp-alignment-margin}. The resulting upper bound is
\begin{equation}
 -s_\ell\left(\gamma_\ell-\frac{M_\ell s_\ell}{2}\right)
 \norm{Q_\ell}_F^2,
\end{equation}
which is negative under~\eqref{eq:supp-step-bound}.
\end{proof}

When both matrices are nonzero, the diagnostic cosine is
\begin{equation}
 \alpha_\ell
 =\frac{\fip{\mathcal G_\ell}{Q_\ell}}
 {\norm{\mathcal G_\ell}_F\norm{Q_\ell}_F}.
\end{equation}
Positive cosine makes $-Q_\ell$ a first-order descent direction, but the admissible step also depends on relative magnitude and local curvature.

\begin{corollary}[Relative feedback error]
Suppose $\mathcal G_\ell\neq0$ and
$\norm{Q_\ell-\mathcal G_\ell}_F\leq
\varepsilon\norm{\mathcal G_\ell}_F$ for $0\leq\varepsilon<1$.
Then the virtual step is strictly descending whenever
\begin{equation}
 0<s_\ell<
 \frac{2(1-\varepsilon)}
 {M_\ell(1+\varepsilon)^2}.
\end{equation}
\end{corollary}

\begin{proof}
Cauchy--Schwarz and the triangle inequality give
\begin{align}
 \fip{\mathcal G_\ell}{Q_\ell}
 &\geq(1-\varepsilon)\norm{\mathcal G_\ell}_F^2,\\
 \norm{Q_\ell}_F
 &\leq(1+\varepsilon)\norm{\mathcal G_\ell}_F.
\end{align}
Substitution into~\eqref{eq:supp-descent-lemma} proves the claim.
\end{proof}

\section{One-Shot Calibration of the Feedback Map}

This section derives Eq.~(5) in ``The Feedback Operator.'' It shows the exact ridge solution used during Stage~I and the precise effect of the subsequent rank-$k$ truncation.

The probe loss uses the same masking, weighting, and reduction convention as the adaptation loss. One backward pass gives the exact hidden-state gradients in~\eqref{eq:supp-full-credit}. Select the $N_p$ probe positions used by the calibration procedure and stack their top errors and exact gradient columns as
\begin{equation}
 \widetilde E=[e^{(1)},\ldots,e^{(N_p)}]\in\R^{d\times N_p},
 \qquad
 \widetilde G_\ell=[g_\ell^{(1)},\ldots,g_\ell^{(N_p)}]\in\R^{d\times N_p}.
 \label{eq:supp-calibration-pairs}
\end{equation}
Because of~\eqref{eq:supp-cross-token-credit}, a column of $\widetilde G_\ell$ may already aggregate credit from several later losses. The calibration therefore fits a predictor from a same-position top error to that exact gradient column; it does not identify a $d\times d$ sequence-level Jacobian.

The alignment diagnostic in the main paper uses the stacked prediction and exact probe gradient for each block. The value $-0.005$ is the mean blockwise cosine for random feedback, while $0.3$--$0.5$ is the observed range after fitting. These are in-probe diagnostics, not guarantees for later iterations.

The full ridge fit is
\begin{equation}
 \widehat F_\ell
 =\arg\min_{F\in\R^{d\times d}}
 \norm{F\widetilde E-\widetilde G_\ell}_F^2
 +\beta\norm{F}_F^2,
 \qquad \beta>0.
 \label{eq:supp-feedback-ridge}
\end{equation}

\begin{theorem}[Feedback ridge solution]
Problem~\eqref{eq:supp-feedback-ridge} is strictly convex and has the unique solution
\begin{equation}
 \widehat F_\ell
 =\widetilde G_\ell\widetilde E^\top
 \bigl(\widetilde E\widetilde E^\top+\beta I_d\bigr)^{-1}.
 \label{eq:supp-feedback-primal}
\end{equation}
Equivalently,
\begin{equation}
 \widehat F_\ell
 =\widetilde G_\ell
 \bigl(\widetilde E^\top\widetilde E+\beta I_{N_p}\bigr)^{-1}
 \widetilde E^\top.
 \label{eq:supp-feedback-dual}
\end{equation}
\end{theorem}

\begin{proof}
The gradient of the objective is
\begin{equation}
 2(F\widetilde E-\widetilde G_\ell)\widetilde E^\top+2\beta F.
\end{equation}
The first-order condition is
\begin{equation}
 F(\widetilde E\widetilde E^\top+\beta I_d)
 =\widetilde G_\ell\widetilde E^\top.
\end{equation}
For nonzero $x$,
\begin{equation}
 x^\top(\widetilde E\widetilde E^\top+\beta I_d)x
 =\norm{\widetilde E^\top x}^2+\beta\norm{x}^2>0,
\end{equation}
so the matrix is positive definite, proving uniqueness and~\eqref{eq:supp-feedback-primal}. The push-through identity
\begin{equation}
 \widetilde E^\top(\widetilde E\widetilde E^\top+\beta I_d)^{-1}
 =(\widetilde E^\top\widetilde E+\beta I_{N_p})^{-1}\widetilde E^\top
\end{equation}
gives~\eqref{eq:supp-feedback-dual}. In implementation, these expressions are evaluated by solving a linear system rather than forming an explicit inverse.
\end{proof}

Let $\widehat F_\ell=U\Sigma V^\top$, with
$\sigma_1\geq\sigma_2\geq\cdots\geq0$. The map retained by LoCA is
\begin{equation}
 F_\ell=U_k\Sigma_kV_k^\top,
 \qquad k=8\ \text{in the experiments}.
 \label{eq:supp-rank-k}
\end{equation}

\begin{lemma}[Rank-truncation error]
With the convention $\sigma_{k+1}=0$ when $k$ is at least the rank of $\widehat F_\ell$,
\begin{align}
 \norm{\widehat F_\ell-F_\ell}_F^2
 &=\sum_{j>k}\sigma_j^2,\\
 \norm{(\widehat F_\ell-F_\ell)e}
 &\leq\sigma_{k+1}\norm{e}.
\end{align}
\end{lemma}

\begin{proof}
The first identity is the Eckart--Young truncated-SVD property. The second follows from
$\norm{\widehat F_\ell-F_\ell}_2=\sigma_{k+1}$.
\end{proof}

The truncation is the best rank-$k$ Frobenius approximation to the fitted matrix $\widehat F_\ell$. It is not, in general, the exact optimizer of a rank-constrained version of~\eqref{eq:supp-feedback-ridge}, whose data-fit term is weighted by the probe covariance. Nor does the calibration theorem provide an out-of-distribution or long-horizon alignment guarantee.

\section{Closed-Form Block Fit and Target Realization}

This section expands ``Closed-Form Per-Block Solve'' and Eqs.~(6)--(9) of the main paper. It derives the exact ridge minimizer and separates exact optimization of the local quadratic from exact realization of the virtual hidden-state target.

At outer iteration $t$, let $P_\ell\in\R^{r\times N}$ stack the selected projected features. For position $n$, write
$q_\ell^{(n)}=F_\ell e^{(n)}$ and let $s_\ell^{(n)}$ be the absolute step used for its mini-batch. Its residual target is
\begin{equation}
 \rho_\ell^{(n)}
 =B_\ell^{(t)}p_\ell^{(n)}
 -s_\ell^{(n)}q_\ell^{(n)},
 \qquad
 R_\ell=[\rho_\ell^{(1)},\ldots,\rho_\ell^{(N)}].
 \label{eq:supp-residual-target}
\end{equation}
All positions in one mini-batch share the same step. During the solve, $P_\ell$ and $R_\ell$ are fixed.

The local objective is
\begin{equation}
 \Phi_\ell(B)
 =\norm{BP_\ell-R_\ell}_F^2+\lambda\norm{B}_F^2,
 \qquad \lambda>0.
 \label{eq:supp-block-objective}
\end{equation}
Define
\begin{equation}
 G_\ell:=P_\ell P_\ell^\top\in\R^{r\times r},
 \qquad
 C_\ell:=R_\ell P_\ell^\top\in\R^{d\times r}.
 \label{eq:supp-stats}
\end{equation}

\begin{theorem}[Unique closed-form block minimizer]
Problem~\eqref{eq:supp-block-objective} is strictly convex and has the unique minimizer
\begin{equation}
 B_\ell^\star=C_\ell(G_\ell+\lambda I_r)^{-1}.
 \label{eq:supp-closed-form}
\end{equation}
Moreover, for every $B\in\R^{d\times r}$,
\begin{equation}
 \Phi_\ell(B)-\Phi_\ell(B_\ell^\star)
 =\tr\left((B-B_\ell^\star)(G_\ell+\lambda I_r)
 (B-B_\ell^\star)^\top\right)\geq0.
 \label{eq:supp-gap-identity}
\end{equation}
The inequality is strict unless $B=B_\ell^\star$.
\end{theorem}

\begin{proof}
Expanding the data-fit term gives
\begin{align}
 \Phi_\ell(B)
 &=\tr(BP_\ell P_\ell^\top B^\top)
 -2\tr(R_\ell P_\ell^\top B^\top)
 +\tr(R_\ell R_\ell^\top)
 +\lambda\tr(BB^\top).
\end{align}
Thus
\begin{equation}
 \nabla_B\Phi_\ell(B)
 =2B(G_\ell+\lambda I_r)-2C_\ell.
\end{equation}
For nonzero $x\in\R^r$,
\begin{equation}
 x^\top(G_\ell+\lambda I_r)x
 =\norm{P_\ell^\top x}^2+\lambda\norm{x}^2>0.
\end{equation}
Therefore $G_\ell+\lambda I_r$ is positive definite, the objective is strictly convex, and the first-order condition gives~\eqref{eq:supp-closed-form}. Substituting
$B=B_\ell^\star+\Delta$ and using
$B_\ell^\star(G_\ell+\lambda I_r)=C_\ell$
gives~\eqref{eq:supp-gap-identity}.
\end{proof}

The desired displacement at fitted position $n$ is
$-s_\ell^{(n)}q_\ell^{(n)}$. Stack these displacements as
$D_\ell^{\mathrm{tar}}$. The fitted adapter realizes
\begin{equation}
 \Delta_\ell^{\mathrm{fit}}
 =(B_\ell^\star-B_\ell^{(t)})P_\ell.
\end{equation}
Their difference is exactly the data residual
\begin{equation}
 E_\ell^{\mathrm{fit}}
 :=\Delta_\ell^{\mathrm{fit}}-D_\ell^{\mathrm{tar}}
 =B_\ell^\star P_\ell-R_\ell.
 \label{eq:supp-realization-error}
\end{equation}
It need not vanish because the target can lie outside the adapter feature span and because ridge regularization trades fit against $\norm{B}_F^2$. Moreover, Section~S4 uses all sequence positions, whereas the implemented fit uses selected prediction positions. This restriction is another approximation.

For an isolated update of $B_\ell$ with its cached input held fixed, the direct change at all sequence positions is
\begin{equation}
 \Delta H_\ell^{\mathrm{dir}}
 =(B_\ell^\star-B_\ell^{(t)})P_\ell^{\mathrm{all}}.
 \label{eq:supp-actual-state-change}
\end{equation}
Under the smoothness assumption of Section~S4,
\begin{equation}
 \mathcal L(H_\ell+\Delta H_\ell^{\mathrm{dir}})
 -\mathcal L(H_\ell)
 \leq
 \fip{\mathcal G_\ell}{\Delta H_\ell^{\mathrm{dir}}}
 +\frac{M_\ell}{2}\norm{\Delta H_\ell^{\mathrm{dir}}}_F^2.
 \label{eq:supp-realized-descent}
\end{equation}
This bound applies to the isolated cached-input change. It is not the full state change after a Jacobi update, because simultaneous upstream updates also change $H_{\ell-1}$ and $P_\ell^{\mathrm{all}}$. Neither this bound nor the ridge theorem guarantees descent of the outer iteration.

\section{Streaming Statistics and Numerical Solution}

This section expands the practical paragraph after Eq.~(9) in ``Closed-Form Per-Block Solve.'' It proves that the exact sufficient statistics can be accumulated by mini-batch and states the numerical solve actually required.

Partition the $N$ fitted positions into mini-batches
$\mathcal I_1,\ldots,\mathcal I_M$. Then
\begin{align}
 G_\ell
 &=\sum_{m=1}^{M}\sum_{n\in\mathcal I_m}
 p_\ell^{(n)}p_\ell^{(n)\top},\\
 C_\ell
 &=\sum_{m=1}^{M}\sum_{n\in\mathcal I_m}
 \rho_\ell^{(n)}p_\ell^{(n)\top}.
 \label{eq:supp-streaming}
\end{align}
Each mini-batch can be discarded after its contribution has been accumulated. For explicit nonnegative weights $a_n$, the same derivation uses
\begin{align}
 G_\ell&=\sum_n a_np_\ell^{(n)}p_\ell^{(n)\top},\\
 C_\ell&=\sum_n a_n\rho_\ell^{(n)}p_\ell^{(n)\top}.
\end{align}
The corresponding objective is
$\sum_na_n\norm{Bp_\ell^{(n)}-\rho_\ell^{(n)}}^2+\lambda\norm{B}_F^2$.
The ridge term keeps the system positive definite even when the observed features do not span $\R^r$.

The persistent statistics require $r^2+dr$ values per block and do not grow with the total number of accumulated tokens. Numerically, LoCA solves
\begin{equation}
 (G_\ell+\lambda I_r)X=C_\ell^\top,
 \qquad B_\ell^\star=X^\top,
 \label{eq:supp-linear-solve}
\end{equation}
using a positive-definite linear solver. No explicit matrix inverse is required.

\section{Scale-Normalized Targets and the Outer Iteration}

This section expands ``Full Algorithm'' and Eq.~(10) of the main paper. It defines RMS precisely, distinguishes the dimensionless target coefficient from the absolute step used in Section~S4, and explains why exact block solves do not imply convergence of the outer loop.

For a nonempty matrix $X\in\R^{a\times b}$, define
\begin{equation}
 \rms(X):=\frac{\norm{X}_F}{\sqrt{ab}}.
\end{equation}
The implementation uses a separate scalar for each block and mini-batch. For one such batch, let $\eta\geq0$ be the dimensionless candidate selected by held-out CE and define
\begin{equation}
 \bar\eta_\ell
 =
 \begin{cases}
 \displaystyle
 \eta\frac{\rms(H_{\ell-1})}{\rms(Q_\ell)},
 & \rms(Q_\ell)>0,\\[1.2ex]
 0,&Q_\ell=0,
 \end{cases}
 \qquad
 D_\ell^{\mathrm{tar}}=-\bar\eta_\ell Q_\ell.
 \label{eq:supp-normalized-correction}
\end{equation}
Here $\bar\eta_\ell$ is the absolute step $s_\ell$ used in Section~S4. When Section~S6 accumulates several mini-batches, each target column uses the step computed for its own mini-batch.

\begin{proposition}[Target magnitude and positive scale invariance]
If $Q_\ell\neq0$, then
\begin{equation}
 \rms(D_\ell^{\mathrm{tar}})
 =\eta\,\rms(H_{\ell-1}).
 \label{eq:supp-relative-size}
\end{equation}
Replacing $Q_\ell$ by $aQ_\ell$ for any $a>0$ leaves
$D_\ell^{\mathrm{tar}}$ unchanged.
\end{proposition}

\begin{proof}
Absolute homogeneity of RMS gives
\begin{equation}
 \rms(D_\ell^{\mathrm{tar}})
 =\bar\eta_\ell\rms(Q_\ell)
 =\eta\rms(H_{\ell-1}).
\end{equation}
For $a>0$, both the denominator of $\bar\eta_\ell$ and $Q_\ell$ acquire the same factor $a$, which cancels.
\end{proof}

Equation~\eqref{eq:supp-relative-size} concerns the virtual target, not the direct cached-input change in~\eqref{eq:supp-actual-state-change}. Moreover, the sufficient descent bound in~\eqref{eq:supp-step-bound} applies to $\bar\eta_\ell$, not directly to the shared coefficient $\eta$. Scale normalization supports transfer of a candidate range; it is not a descent theorem and does not make LoCA hyperparameter-free.

Let $\mathcal D(B^{(t)})$ denote all local data produced by the current forward pass, including $\{P_\ell,\bar H_\ell,Q_\ell,R_\ell\}_{\ell=1}^{L}$, and let $\mathcal S$ apply the blockwise ridge solutions. A Jacobi outer update is
\begin{equation}
 B^{(t+1)}
 =\mathcal S\bigl(\mathcal D(B^{(t)})\bigr).
 \label{eq:supp-outer-map}
\end{equation}
Theorem~S6.1 proves that every component of $\mathcal S$ minimizes its current local quadratic. It does not imply that $\mathcal S\circ\mathcal D$ is a contraction. After an update, the projected features, frozen-block outputs, top error, local targets, and realization errors can all change. Consequently, the next iteration solves different quadratics and global CE can be non-monotone. The main paper therefore selects the returned snapshot by held-out CE and retains the frozen model as a candidate.

A Jacobi schedule forms all local data from one shared forward pass and then solves blocks independently. A Gauss--Seidel schedule refreshes downstream states after a block update and therefore requires additional partial forward computation. No convergence theorem is claimed for either schedule on the nonlinear adapted transformer.

\section{Experimental Setup and Recovery Metrics}

This section corresponds to ``Experiments--Setup.'' It records the comparison boundary, explains the recovery ratios, and states which experimental conclusions are descriptive.

The main comparison uses the frozen model, LoRA, Adapter-MeZO, and LoCA. Adapter-MeZO perturbs the same low-rank adapter parameterization as LoCA and LoRA; full-parameter MeZO is reported separately in Section~S13. The primary reported metrics are per-token cross-entropy (CE) and ranking accuracy.

The recovery ratios in the main paper are
\begin{align}
 R_{\mathrm{CE}}
 &=
 \frac{\mathrm{CE}_{\mathrm{frozen}}-\mathrm{CE}_{\mathrm{LoCA}}}
 {\mathrm{CE}_{\mathrm{frozen}}-\mathrm{CE}_{\mathrm{LoRA}}},\\
 R_{\mathrm{acc}}
 &=
 \frac{\mathrm{acc}_{\mathrm{LoCA}}-\mathrm{acc}_{\mathrm{frozen}}}
 {\mathrm{acc}_{\mathrm{LoRA}}-\mathrm{acc}_{\mathrm{frozen}}}.
 \label{eq:supp-recovery}
\end{align}
A value of one matches the reported LoRA value. When the frozen-to-LoRA denominator is positive, a value above one means that LoCA improves more than LoRA in that cell. The ratio is unstable near a zero denominator, so such cells are flagged or omitted and raw CE and accuracy remain primary.

LoCA uses adapter rank $r=32$, at most $40$ outer iterations, and held-out CE for early stopping and for selecting $\eta$ from the stated candidate set. Unless otherwise noted, table cells are single runs at seed~0. Cell-level differences and counts are therefore descriptive rather than estimates of population-level superiority.

\section{Experiment 1: Main Benchmark Comparison}

This section corresponds to ``Experiments--Main Results'' and Table~1. It explains the two observations that need qualification: the count of lower-CE cells and the CE--accuracy discrepancy on BoolQ.

Across the $25$ reported Qwen2.5 task--scale cells, LoCA has lower evaluation CE than the corresponding LoRA run in $16$ cells. This count shows that the local ridge path can differ materially from the LoRA optimization path. Because most cells are single runs, it is not evidence that LoCA is generally superior to LoRA.

For cells with a positive frozen-to-LoRA denominator, recovery above one means that LoCA improves more than the reported LoRA run. It is not a significance test, and it can become large when the denominator is small. The main paper therefore also reports the raw metrics.

On BoolQ at 7B and 14B, the LoRA runs reduce CE while ranking accuracy falls from $0.824$ to $0.636$ and from $0.852$ to $0.364$. The LoCA runs reach $0.840$ and $0.864$. This establishes a mismatch between token-level CE and answer ranking in those runs. Completion-format overfitting is a plausible interpretation, but the experiment does not isolate it causally.

\section{Experiment 2: Resource Use}

This section corresponds to ``Experiments--Resource Use'' and Table~2. It explains the measurement boundary and connects the observed storage to the sufficient statistics derived in Section~S7.

For each block and $N$ fitted positions, the principal Stage~II operations are
\begin{align}
 P_\ell=A_\ell H_{\ell-1}&: &&O(Ndr),\\
 Q_\ell=F_\ell E&: &&O(Ndk),\\
 G_\ell=P_\ell P_\ell^\top&: &&O(Nr^2),\\
 C_\ell=R_\ell P_\ell^\top&: &&O(Ndr),\\
 B_\ell^\star&: &&O(r^3+dr^2).
\end{align}
The low-rank feedback cost uses
$F_\ell e=U_k(\Sigma_k(V_k^\top e))$. The exact head error additionally applies $W_u^\top$ once per top error; this shared head computation is not multiplied by the number of blocks.

Beyond the frozen checkpoint and ordinary inference buffers, the persistent Stage~II storage is
\begin{equation}
 O\bigl(Lr^2+Ldr+Ldk\bigr),
\end{equation}
up to constant factors for the fixed projections, adapters, and two feedback factors. The sufficient-statistic term does not grow with the total number of processed tokens, although batch size and sequence length still affect transient forward buffers.

The GPU peak in Table~2 is measured over the full LoCA run and includes the temporary backward graph used in Stage~I. The CPU steady-state memory and per-pass time describe Stage~II after calibration. These are different measurement boundaries and should not be combined into a claim that the complete method is forward-only or that calibration has zero cost.

Adapter-MeZO has the lowest memory in the table because it stores neither a backward graph nor ridge statistics. In the reported runs, it instead uses $10^3$--$10^4$ perturbation steps. LoCA uses at most $40$ outer iterations, but still pays for calibration and for selecting $\eta$ and $\lambda$. The measurements support the stated memory and per-pass comparisons; because calibration latency was not separately recorded, they do not establish a universal end-to-end speed advantage.

\section{Experiment 3: Scale-Normalized Cross-Family Transfer}

This section corresponds to ``Experiments--Cross-Family Results'' and Table~3. It explains exactly what transfers across model sizes and what is still selected on held-out data.

The normalized experiment reuses the candidate set
\begin{equation}
 \eta\in\{0.003,0.01,0.03\}
\end{equation}
for each model reported in Table~3 of the main paper. A value is still selected from this set by held-out CE. Thus the candidate range is reused without redesign, but one fixed value is not used without validation.

On Qwen2.5 SST-2, the normalized candidate set yields the recoveries reported in Table~3 across four model sizes. The comparison with separately tuned absolute coefficients tests whether relative residual-stream change is more portable than an absolute hidden-state correction. It does not prove scale invariance outside the tested models.

The SmolLM2-1.7B rows provide the cross-family check. Reusing the absolute coefficient gives recovery between $0$ and $0.07$ on the four reported tasks, whereas the normalized target gives recovery between $0.93$ and $1.29$. This supports the relative target parameterization on a second family. It does not make the method hyperparameter-free or establish transfer to arbitrary architectures.

\section{Precise Boundary of the Forward-Only Claim}

After calibration, assume that the following objects are available on the adaptation device: the frozen checkpoint, fixed projections $A_\ell$, low-rank feedback factors for $F_\ell$, and the current adapters $B_\ell$. Then one outer iteration consists of:
\begin{enumerate}[leftmargin=2em]
 \item a forward pass through the adapted frozen model;
 \item evaluation of~\eqref{eq:supp-top-error} at the frozen head;
 \item application of $F_\ell e$ using fixed low-rank factors;
 \item streaming accumulation of $G_\ell$ and $C_\ell$;
 \item positive-definite linear solves for $B_\ell$.
\end{enumerate}
None of these steps differentiates through a backbone block. Therefore the adaptation loop is backward-free after calibration. The complete method is not backpropagation-free, because producing the calibrated feedback maps requires the one probe backward pass described in Section~S5. A backward-capable host may perform calibration and transfer the resulting maps to an inference-oriented adaptation device, provided the checkpoint and calibration assumptions match.

\section{MeZO Comparison and Full-Parameter Context}

This section corresponds to the Adapter-MeZO discussion in ``Related Work,'' ``Experiments--Setup,'' and the Discussion section. It explains why Adapter-MeZO is the matched baseline and reports the separate full-parameter MeZO scaling observation used in the main paper.

Adapter-MeZO is the main comparison because it perturbs the same low-rank adapter parameterization as LoCA and LoRA. The methods therefore act on comparable parameter spaces even though they obtain their updates differently: Adapter-MeZO uses random function-value perturbations, whereas LoCA solves the current local ridge objectives. In the reported runs, Adapter-MeZO has the smallest memory footprint but uses $10^3$--$10^4$ perturbation steps and a learning-rate search at each scale.

\begin{figure}[H]
  \centering
  \includegraphics[width=0.95\textwidth,trim=0 0 0 25,clip]{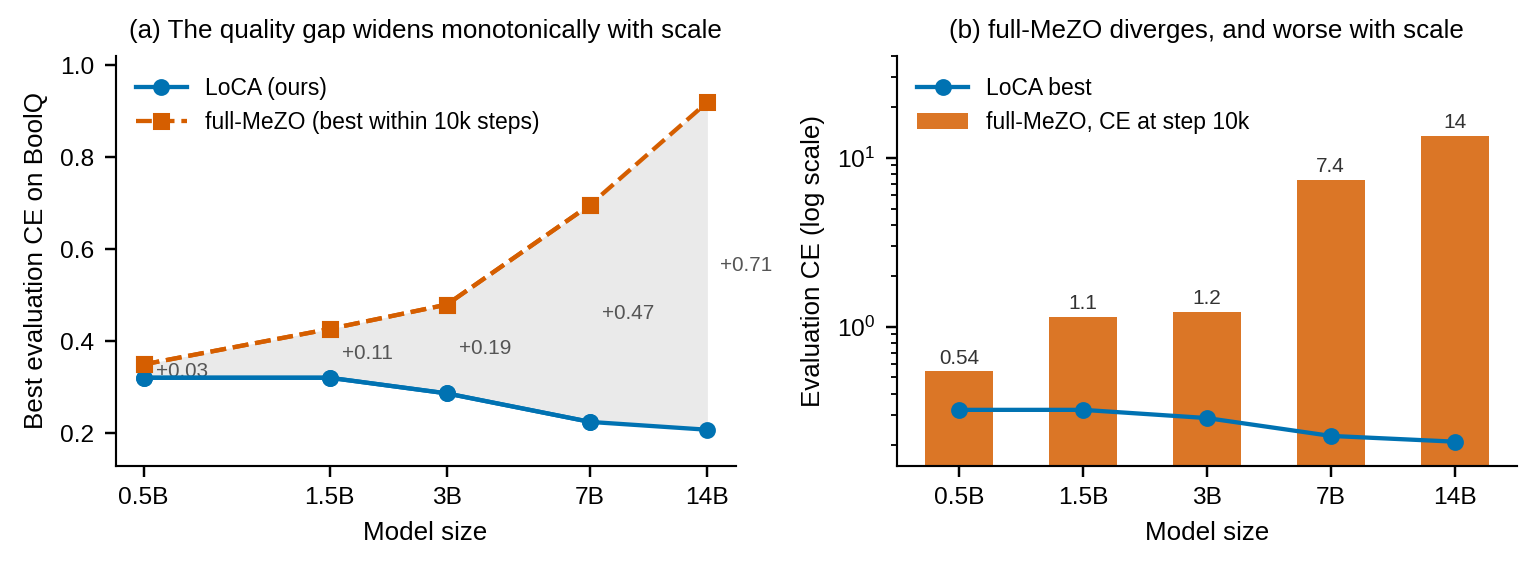}
  \caption{Full-parameter MeZO on BoolQ across the tested model sizes. Left: best evaluation CE within the reported runs. Right: evaluation CE at step 10{,}000. The full-parameter result worsens with scale in this setting, motivating the matched Adapter-MeZO baseline in the main benchmark.}
  \label{fig:full_mezo}
\end{figure}

Figure~\ref{fig:full_mezo} is limited to BoolQ on the tested Qwen2.5 sizes. It does not establish algorithmic divergence, a long-context result, or a general claim about all tasks or zeroth-order methods.

\end{document}